\documentclass{article}

\usepackage[preprint]{neurips_2026}
\setcitestyle{numbers,square,sort&compress}

\usepackage[utf8]{inputenc}
\usepackage[T1]{fontenc}
\usepackage{microtype}
\usepackage{nicefrac}

\usepackage{amsmath}
\usepackage{amsfonts}
\usepackage{amssymb}
\usepackage{amsthm}

\usepackage[table]{xcolor}
\definecolor{antisdrow}{rgb}{0.92, 0.96, 0.92}  
\usepackage{graphicx}
\graphicspath{{figures/}}
\usepackage{booktabs}
\usepackage{multirow}
\usepackage{enumitem}
\usepackage{algorithm}
\usepackage{algpseudocode}
\usepackage[breakable]{tcolorbox}  
\usepackage{wrapfig}               
\usepackage{fontawesome5}          

\usepackage{url}

\usepackage{hyperref}

\newtheorem{theorem}{Theorem}

\newtheorem{lemma}[theorem]{Lemma}

\newcommand{\resourcelink}[3]{%
  #1~\texttt{\href{#3}{#2}}%
}

\title{Anti-Self-Distillation for Reasoning RL via\\ Pointwise Mutual Information}

\author{%
  Guobin Shen$^{1}$  \quad Xiang Cheng$^{1}$ \quad Chenxiao Zhao$^{1}$ \quad Lei Huang$^{1}$\\[2pt]
  \textbf{Jindong Li$^{2}$ \quad Dongcheng Zhao$^{2}$ \quad Xing Yu$^{1}$}\thanks{Correspondence to: \texttt{yuanshan2@xiaohongshu.com}, \texttt{floyed\_shen@outlook.com}.} \\[4pt]
  $^{1}$Xiaohongshu Inc. \quad $^{2}$Institute of Automation, Chinese Academy of Sciences
}

\begin{document}

\maketitle

\begin{abstract}
  On-policy self-distillation, where a student is pulled toward a copy of itself conditioned on privileged context (e.g., a verified solution or feedback), offers a promising direction for advancing reasoning capability without a stronger external teacher.
  Yet in math reasoning the gains are inconsistent, even when the same approach succeeds elsewhere.
  A pointwise mutual information analysis traces the failure to the privileged context itself: it inflates the teacher's confidence on tokens already implied by the solution (structural connectives, verifiable claims) and deflates it on deliberation tokens (\emph{Wait}, \emph{Let}, \emph{Maybe}) that drive multi-step search.
  We propose Anti-Self-Distillation (AntiSD), which ascends a divergence between student and teacher rather than descending it: this reverses the per-token sign and yields a naturally bounded advantage in one step.
  An entropy-triggered gate disables the term once the teacher entropy collapses, completing a drop-in replacement for default self-distillation.
  Across five models from 4B to 30B parameters on math reasoning benchmarks, AntiSD reaches the GRPO baseline's accuracy in $2$ to $10\times$ fewer training steps and improves final accuracy by up to $11.5$ points.
  AntiSD opens a path to scalable self-improvement, where a language model bootstraps its own reasoning through its training signal.
\end{abstract}

\begin{center}\vspace{-0.6em}\small
  \resourcelink{\faGithub}{github.com/FloyedShen/AntiSD}{https://github.com/FloyedShen/AntiSD}%
  \hspace{2em}%
  \resourcelink{\faChartLine}{wandb.ai/brain-cog/AntiSD}{https://wandb.ai/brain-cog/AntiSD}
\end{center}

\section{Introduction}
\label{sec:intro}
\vspace{-0.5em}

Reinforcement learning has become a primary axis of post-training progress for reasoning tasks, with reinforcement learning from verifiable rewards (RLVR; \citealp{grpo,dapo,deepseek-r1,team2025kimi}) emerging as the dominant paradigm. The reward signal in RLVR, however, is typically a sparse, trajectory-level scalar: a single bit per rollout that does not indicate which intermediate step was responsible, leaving credit assignment to individual reasoning steps as an open problem. To address this, two main directions have emerged: training a separate process reward model (PRM) to score intermediate steps \citep{prm800k,math-shepherd,luo2024improve}, or applying on-policy distillation (OPD) to provide a token-level imitation signal from a stronger teacher \citep{gkd,fu2026revisitingonpolicydistillationempirical, lu2025onpolicydistillation}. Both, however, depend on an external model. Can the model itself supply this credit?

On-policy self-distillation answers this in the affirmative. It specializes OPD by taking the teacher to be the student itself, conditioned on privileged context: typically a verified solution and any feedback from the environment. The token-level signal is then produced by the model's own forward pass under richer conditioning, requiring neither an external teacher nor a separate reward model. A series of recent works \citep{opsd,sdpo,opcd-ye,opsdc} has developed this idea along several axes, connecting back to the older framework of learning under privileged information \citep{vapnik-lupi,lopez2015unifying}.

On math reasoning, however, the picture is more mixed. Diagnostic studies report that on-policy self-distillation can improve instruction-following, scientific QA, and tool-use tasks \citep{sdpo}, while delivering only modest or inconsistent gains on more challenging mathematical problems \citep{kim2026does}. We observe the same pattern across model families ranging from 4B to 30B parameters: on math reasoning benchmarks such as AIME 2024 and 2025, default self-distillation typically fails to outperform a strong GRPO baseline (Figure~\ref{fig:fig1}\,(b) shows one representative case; full sweep in Section~\ref{sec:exp:main}).

\begin{figure}[t]
  \centering
  \begin{minipage}[c]{0.62\linewidth}
    \centering
    \includegraphics[width=\linewidth]{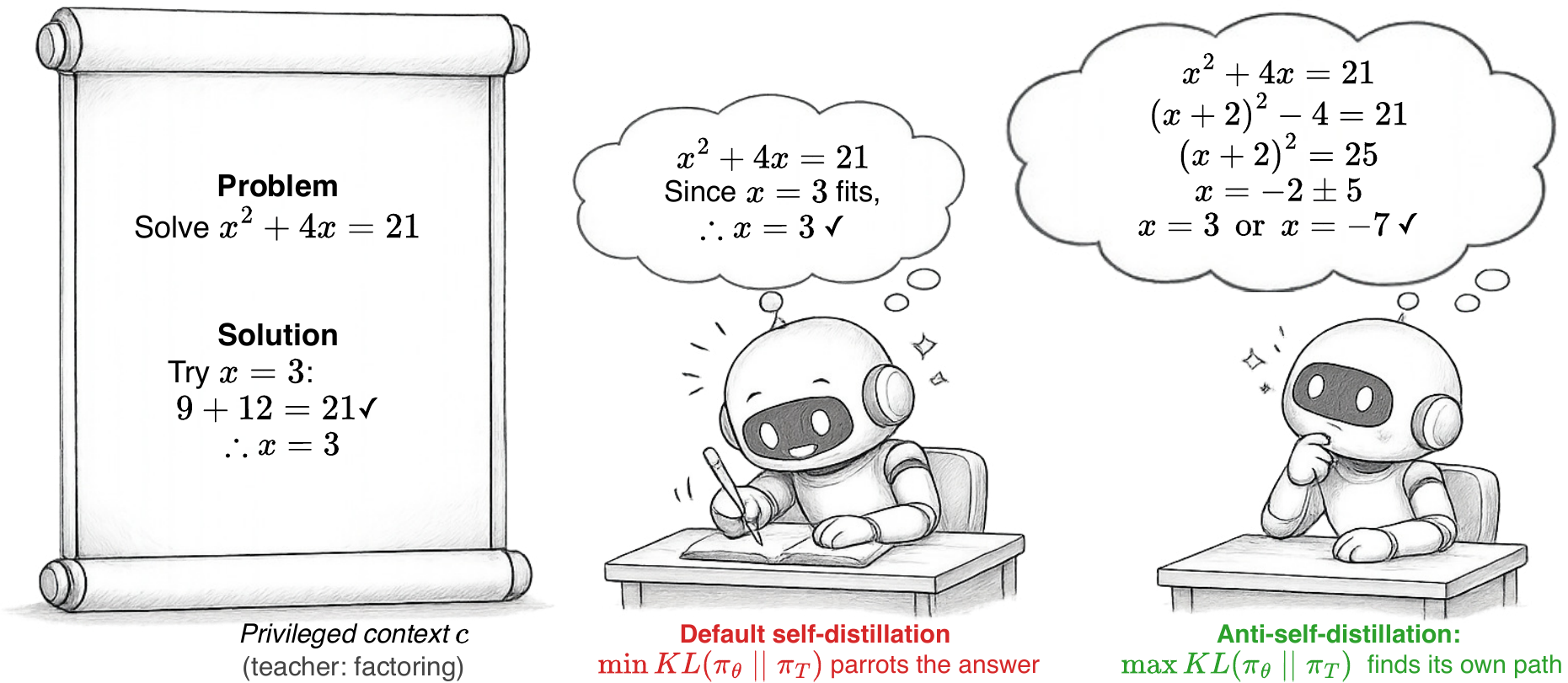}
  \end{minipage}\hfill
  \begin{minipage}[c]{0.36\linewidth}
    \centering
    \includegraphics[width=\linewidth]{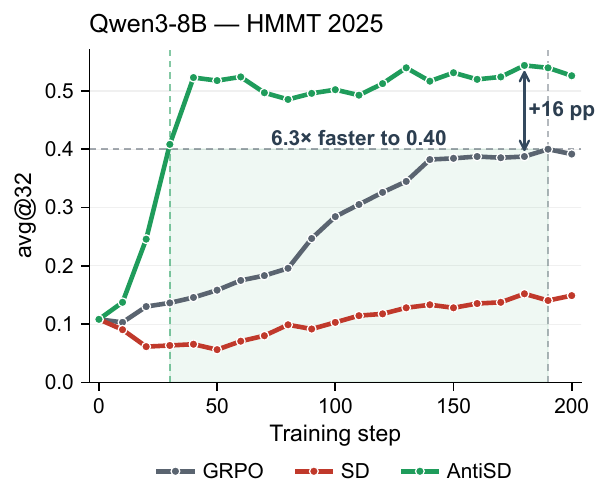}
  \end{minipage}
  \caption{\textbf{(a)} An oracle-conditioned teacher biases the student toward a single root; reversing the signal preserves deliberation and recovers the rest. \textbf{(b)} Qwen3-8B on HMMT 2025: AntiSD reaches GRPO's peak in $\sim$$1/5$ the steps and ends $+15$\,pp higher; default self-distillation underperforms GRPO.}
  \label{fig:fig1}
  \vspace{-1em}
\end{figure}

To understand the cause, we inspect the per-token signal that default self-distillation produces (Figure~\ref{fig:fig2}). The pattern points to the privileged context itself: conditioning the teacher on a verified solution effectively turns it into an oracle, leaving it confident on tokens that follow once the answer is known, such as structural connectives and verifiable-claim words, and unsure on deliberation tokens like \emph{Wait}, \emph{Let}, and \emph{Maybe} that the student emits when re-examining alternatives. Standard self-distillation pulls the student toward this oracle teacher, reinforcing tokens that track the known solution and weakening tokens that drive deliberation, as shown in Figure~\ref{fig:fig1}\,(a).

This motivates a simple fix: invert the gradient direction. We propose \emph{Anti-Self-Distillation} (AntiSD), which ascends a divergence between student and teacher rather than descending it, reversing the per-token sign and yielding a naturally bounded advantage in one step. An entropy-triggered gate disables the term once the teacher's per-token entropy collapses, completing a drop-in replacement for default self-distillation. Across five models from 4B to 30B parameters on math reasoning benchmarks, AntiSD reaches the GRPO baseline's accuracy in $2$ to $10\times$ fewer training steps and improves final accuracy by up to $11.5$ points.

Our contributions are summarized as follows:
\begin{itemize}[leftmargin=10pt,itemsep=-1pt]
  \item We expose a structural shortcut bias in standard self-distillation, where the per-token signal rewards tokens the privileged context already implies and suppresses deliberation tokens, and ground this observation in a conditional pointwise mutual information identity (Section~\ref{sec:antisd:direction}).
  \item We propose \emph{Anti-Self-Distillation} (AntiSD), which reverses the per-token signal by ascending Jensen-Shannon divergence between student and teacher; the JSD shape provides automatic bounding, leaving an entropy-triggered gate as the only practical stabilizer. AntiSD is a drop-in replacement for default self-distillation with no additional cost.
  \item Across five 4B--30B models on math and coding tasks, AntiSD matches the GRPO baseline in $2$ to $10\times$ fewer steps and adds up to $11.5$ points of final accuracy over both GRPO and default self-distillation.
\end{itemize}

\begin{figure}[t]
  \centering
  \begin{minipage}[t]{0.62\linewidth}
    \centering
    \includegraphics[width=\linewidth]{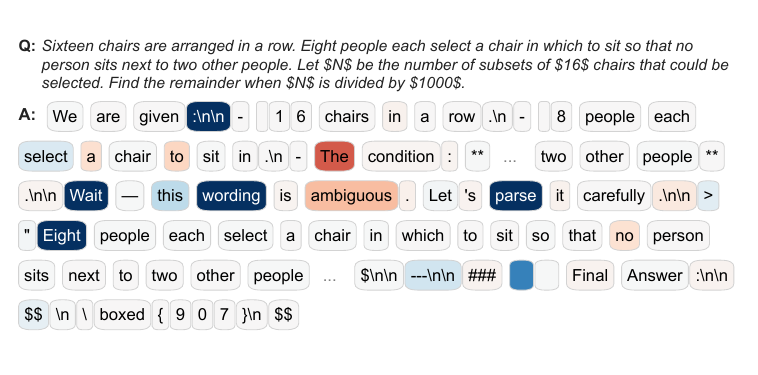}\\
  \end{minipage}\hfill
  \begin{minipage}[t]{0.38\linewidth}
    \centering
    \includegraphics[width=\linewidth]{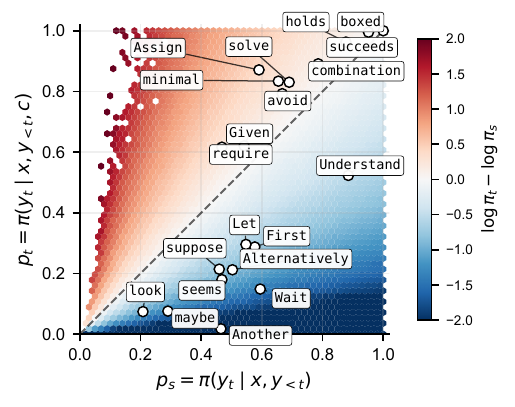}\\
  \end{minipage}
  \caption{Per-token signal $u_t = t_t - s_t$ on Qwen3-4B-IT-2507 at AIME-25. \textbf{(a)}~Single-rollout trace. \textbf{(b)}~$(\pi_S, \pi_T)$ heatmap. Blue marks deliberation tokens ($u_t \ll 0$); red marks shortcut tokens ($u_t \gg 0$).}
  \label{fig:fig2}
  \vspace{-1.5em}
\end{figure}
\section{Preliminaries}
\label{sec:prelim}
\vspace{-.5em}

\textbf{Setup.} We work with an autoregressive language model $\pi_\theta$ that, given a problem $x$, samples a trajectory $y = (y_1, \ldots, y_T)$. RLVR provides a scalar verifiable reward $R(x, y)$ scoring the final answer. Following GRPO \citep{grpo}, we sample a group of $G$ rollouts per prompt and use the group-normalized sequence-level advantage $A_i^{\mathrm{seq}} := (R_i - \mu_R)/\sigma_R$ as the policy-gradient signal for the $i$-th rollout, where $\mu_R$ and $\sigma_R$ are the within-group mean and standard deviation.

\textbf{On-policy self-distillation.} On-policy self-distillation augments the GRPO objective with a per-token signal derived from a self-teacher. Let $c$ denote privileged context (a verified solution and any environment feedback) provided at training time but not at inference. The same network $\pi_\theta$ plays two roles: the student $\pi_S(\cdot \mid x, y_{<t}) := \pi_\theta(\cdot \mid x, y_{<t})$ generates the rollout, while the teacher $\pi_T(\cdot \mid x, y_{<t}) := \pi_\theta(\cdot \mid x, c, y_{<t})$ scores it under richer conditioning (we suppress $c$ from the teacher's left-hand-side conditioning as a notational shorthand, since $c$ is fixed throughout each training step). With $\texttt{sg}[\cdot]$ denoting stop-gradient, the standard self-distillation loss is the per-token KL,
\begin{equation}
  \label{eq:sd}
  \mathcal{L}_{\mathrm{SD}}(\theta)
  \;=\;
  \mathbb{E}_{x,\, y \sim \pi_S(\cdot \mid x)}
  \Big[ \sum_{t=1}^{T} D_{\mathrm{KL}}\!\big( \pi_S(\cdot \mid x, y_{<t}) \,\big\|\, \texttt{sg}[\pi_T(\cdot \mid x, y_{<t})] \big) \Big],
\end{equation}
in addition to the GRPO objective. More generally, $\mathcal{L}_{\mathrm{SD}}$ is one member of a family of per-token f-divergences between student and teacher; the choice of $f$ shapes the resulting per-token advantage and we revisit it in Section~\ref{sec:antisd:stab}. The basic on-policy distillation formulation drops the GRPO term ($A_i^{\mathrm{seq}} \equiv 0$) and uses this per-token signal alone \citep{gkd,fu2026revisitingonpolicydistillationempirical,opsd}; recent reasoning RL methods \citep{sdpo,li2026unifying,xiao2026mimo} instead combine it with the trajectory-level reward $A_i^{\mathrm{seq}}$ through various forms (additive, multiplicative, or sample-level routing). We adopt the additive form:
\begin{equation}
  \label{eq:adv-combined}
  A_{i,t} \;=\; A_i^{\mathrm{seq}} \;+\; \lambda \cdot \delta_t,
\end{equation}
where $\delta_t$ is the per-token contribution of $-\nabla_\theta \mathcal{L}_{\mathrm{SD}}$ written in policy-gradient form (closed form in Section~\ref{sec:antisd:direction}) and $\lambda > 0$ is a mixing weight.

\section{Anti-Self-Distillation}
\label{sec:antisd}
\vspace{-.5em}

Section~\ref{sec:antisd:direction} identifies the per-token signal $\delta_t$ from Equation~\eqref{eq:adv-combined} with conditional pointwise mutual information and shows, in conjunction with Figure~\ref{fig:fig2}, that it carries a structural shortcut bias. Section~\ref{sec:antisd:stab} responds with \emph{Anti-Self-Distillation} (AntiSD), which ascends Jensen-Shannon divergence between student and teacher under a single entropy-triggered gate.

\subsection{Per-token reward as conditional PMI}
\label{sec:antisd:direction}
\vspace{-0.5em}

We abbreviate $s_t := \log \pi_S(y_t \mid x, y_{<t})$, $t_t := \log \pi_T(y_t \mid x, y_{<t})$, and $u_t := t_t - s_t$. To get a closed form for $\delta_t$ in Equation~\eqref{eq:adv-combined}, differentiate the per-token KL summand $\mathbb{E}_{v \sim \pi_S}[\log \pi_S(v) - \log \pi_T(v)]$ in Equation~\eqref{eq:sd} with respect to $\theta$. The constant-coefficient term $\mathbb{E}_v[\nabla_\theta \log \pi_S(v)]$ vanishes by the score-function identity $\sum_v \nabla_\theta \pi_S(v) = 0$, the teacher gradient is killed by the stop-gradient, and only a weighted score-function term survives (full proof in Appendix~\ref{app:proofs}, Lemma~\ref{lem:rkl-grad}):
\begin{equation}
  \nabla_\theta\, D_{\mathrm{KL}}\!\big( \pi_S(\cdot \mid x, y_{<t}) \,\big\|\, \pi_T(\cdot \mid x, y_{<t}) \big)
  \;=\;
  - \,\mathbb{E}_{v \sim \pi_S(\cdot \mid x, y_{<t})}\!\Big[\, u_v \cdot \nabla_\theta \log \pi_S(v \mid x, y_{<t}) \,\Big],
  \label{eq:thm1-grad}
\end{equation}
with $u_v := \log \pi_T(v \mid x, y_{<t}) - \log \pi_S(v \mid x, y_{<t})$. The combined advantage in Equation~\eqref{eq:adv-combined} therefore uses $\delta_t = +u_t$. Following standard policy-gradient practice for distillation, we treat the outer rollout expectation $\mathbb{E}_{y \sim \pi_S}$ as a sample-mean estimator with stop-gradient on the trajectory distribution; the trajectory-level REINFORCE term that would otherwise arise from differentiating $\pi_S(y \mid x)$ is dropped, since trajectory-level credit assignment is handled separately by the GRPO term $A_i^{\mathrm{seq}}$.

\textbf{$u_t$ as conditional PMI.} Under the self-distillation setup, $\pi_S$ and $\pi_T$ share parameters, so $u_t$ admits a closed-form interpretation:
\begin{equation}
  u_t \;=\; \log \frac{\pi_\theta(y_t \mid x, c, y_{<t})}{\pi_\theta(y_t \mid x, y_{<t})} \;=\; \mathrm{PMI}(y_t\,;\, c \mid x, y_{<t}),
  \label{eq:pmi}
\end{equation}
the conditional pointwise mutual information between the next token $y_t$ and the privileged context $c$. The sign of $u_t$ records whether $c$ raises ($u_t > 0$) or lowers ($u_t < 0$) $\pi_\theta(y_t)$. The default per-token reward $\delta_t = +u_t$ therefore rewards tokens whose probability is raised by $c$ and penalizes those it lowers; Figure~\ref{fig:fig2} makes this concrete on real data.

We compute $u_t$ on student rollouts from Qwen3-4B-IT-2507 at AIME-25, with $c$ from our self-distillation pipeline (Appendix~\ref{app:prompts}). The teacher reward splits tokens into two informative regimes. \emph{Shortcut tokens} ($u_t \gg 0$, deep red) -- \emph{Given}, \emph{Assign}, \emph{succeeds}, \emph{holds} -- are strongly rewarded once the answer is known. \emph{Deliberation tokens} ($u_t \ll 0$, deep blue) -- \emph{Wait}, \emph{Let}, \emph{Maybe}, \emph{Alternatively} -- are strongly penalized, since $c$ has committed to a solution and the teacher down-weights tokens that re-examine alternatives. Generic tokens along the diagonal and answer-template tokens near $(\pi_S, \pi_T) \approx (1,1)$ carry no signal. Figure~\ref{fig:fig2}(a) traces these regimes alternating along a single rollout, and (b) aggregates them into a $(\pi_S, \pi_T)$ heatmap with two off-diagonal lobes of opposite $u_t$ sign.

Default self-distillation thus rewards shortcut tokens and penalizes deliberation tokens. This is consistent with a phenomenon repeatedly observed under on-policy self-distillation -- responses shorten as training proceeds \citep{sdpo,kim2026does,opsdc} -- but recasts it as a structural shortcut rather than benign compression, with the suppression concentrated on the deliberation steps that drive search rather than on redundant filler. The polarity is not specific to reverse KL: for any convex $f$ in the family from Section~\ref{sec:prelim}, descent on $D_f(\pi_S \| \pi_T)$ has per-token advantage monotonically increasing in $u_t$ and inherits the same shortcut/deliberation split.

Two empirical observations from this analysis will drive the method. \emph{(O1) Wrong polarity for reasoning}: the per-token reward $\delta_t = +u_t$ has the wrong sign -- rewarding shortcut tokens and penalizing the deliberation tokens that drive search. \emph{(O2) Asymmetric distribution}: because rollouts come from $\pi_S$, tokens with $\pi_S > \pi_T$ are over-sampled in the batch -- visible in Figure~\ref{fig:fig2}(b) as the heavier deliberation lobe ($u_t < 0$), with individual tokens in the tail reaching $u_t \le -20$ (Figure~\ref{fig:fig2}(a)).

\subsection{Ascent on Jensen-Shannon divergence}
\label{sec:antisd:stab}
\vspace{-0.5em}

AntiSD has three components. From (O1), we \emph{reverse the gradient direction} (descent $\to$ ascent), flipping the per-token reward at the source. From (O2), we \emph{ascend Jensen-Shannon divergence} rather than reverse KL: JSD's f-divergence-derived advantage is asymmetrically bounded (capped on the over-sampled deliberation side and linear on the under-sampled shortcut side), directly counterbalancing the empirical asymmetry. The third component, an \emph{entropy-triggered gate}, follows from the first two: once we ascend a divergence, the policy gradient is no longer self-terminating, so we need a signal-quality criterion to disable the term once $\pi_T$'s information about $\pi_S$ degenerates. We make each concrete below.

\textbf{JSD ascent.}
Writing $D_{\mathrm{JSD}}(\pi_S \| \pi_T) = \mathbb{E}_{\pi_T}[f(\pi_S/\pi_T)]$ for the corresponding f-divergence generator, the score-function trick (analogous to Equation~\eqref{eq:thm1-grad}) gives
\begin{equation}
  \label{eq:fdiv-grad}
  \nabla_\theta D_{\mathrm{JSD}}(\pi_S \| \pi_T) \;=\; \mathbb{E}_{v \sim \pi_S}\!\left[ f'\!\left(\tfrac{\pi_S(v)}{\pi_T(v)}\right) \, \nabla_\theta \log \pi_S(v) \right].
\end{equation}
Substituting $\pi_S/\pi_T = e^{-u}$ identifies $f'(\pi_S/\pi_T) = -\varphi(u)$ (full simplification in Appendix~\ref{app:proofs}), so ascending JSD via policy gradient has per-token advantage
\begin{equation}
  \label{eq:antisd-advantage}
  A_t^{\mathrm{AntiSD}} \;=\; -\varphi(u_t), \qquad \varphi(u) \;:=\; \tfrac{1}{2}\!\left(\mathrm{softplus}(u) - \log 2\right).
\end{equation}
The shape $\varphi$ is the f-divergence derivative for $D_{\mathrm{JSD}}$, so its monotonicity and sign-preservation follow from JSD's convexity. At small $|u|$, $\varphi'(0) = \tfrac{1}{2}\sigma(0) = \tfrac{1}{4}$ gives $-\varphi(u) \approx -\tfrac{1}{4} u$, which recovers ascent on reverse KL up to a positive scalar; the two choices diverge in the tails, where $\varphi(u) \ge -\tfrac{1}{2}\log 2$ globally (proof in Appendix~\ref{app:proofs}) caps the AntiSD advantage on the deliberation side at $\tfrac{1}{2}\log 2$. This is exactly the side that (O2) flagged as both over-sampled and heavy-tailed: the cap absorbs the $u_t \le -20$ spikes and rebalances per-token gradient contributions against the lighter, under-sampled shortcut side, while the shortcut side keeps its linear penalty since extreme shortcut tokens are precisely the ones AntiSD should suppress proportionally. We ablate the divergence choice in Section~\ref{sec:exp:ablation}.

\textbf{Entropy-triggered gate.} The JSD ascent direction is not self-terminating, so we need a criterion to disable the term once $u_t$ stops carrying useful conditional information. The teacher's per-token entropy aggregated over the batch, $H := \mathrm{median}_{i,t}\, H[\pi_T(\cdot \mid x_i, y_{i,<t})]$, provides this signal. The log-ratio $u_t = \log(\pi_T/\pi_S)$ is well-conditioned only as long as $\pi_T$ retains substantial entropy: when $\pi_T$ collapses to a near-deterministic mode (low $H$), most tokens lie at floor probability under $\pi_T$ and $u_v$ becomes dominated by numerical floor rather than conditional information. We disable the AntiSD term when $H$ falls below an auto-calibrated threshold $\tau_{\mathrm{down}}$, and re-enable it once $H$ recovers to its pre-collapse baseline $H_{\mathrm{warm}}$ (a Schmitt trigger to avoid chatter):
\begin{equation}
  g \leftarrow
  \begin{cases}
    1 & \text{if } g=0 \text{ and } H \ge H_{\mathrm{warm}},  \\
    0 & \text{if } g=1 \text{ and } H < \tau_{\mathrm{down}}, \\
    g & \text{otherwise,}
  \end{cases}
  \qquad \lambda = g \cdot \lambda_{\max}.
\end{equation}
$\tau_{\mathrm{down}}$ is auto-calibrated from $W$ warmup steps at $\lambda=0$ (concrete values in Section~\ref{sec:exp} Setup). Algorithm~\ref{alg:antisd} (Appendix~\ref{app:hparams}) summarizes the resulting update.

\section{Experiments}
\label{sec:exp}
\vspace{-0.5em}


\textbf{Setup.} We train five language models from the Qwen3~\citep{yang2025qwen3} and Olmo-3~\citep{olmo2025olmo3} families (4B--30B parameters) on DAPO-Math-17k~\citep{dapo} for 200 on-policy steps, comparing four conditions per model: the un-trained base, +GRPO (Equation~\eqref{eq:adv-combined} with $\lambda=0$), +SD (default self-distillation, $\delta_t = +u_t$), and +AntiSD (Algorithm~\ref{alg:antisd}). The privileged context $c$ is a verified solution sampled from the rollout group when at least one rollout is correct, else from the dataset, concatenated with a binary correctness feedback string. AntiSD's gate is auto-calibrated from the first $5$ training steps (run at $\lambda=0$): we record the median teacher entropy $H_{\mathrm{warm}}$ and set $\tau_{\mathrm{down}} = 0.93\,H_{\mathrm{warm}}$, with the gate re-enabling once $H$ recovers to $H_{\mathrm{warm}}$. The $0.93$ multiplier is shared across all model families, requiring no per-model tuning. Held-out evaluation reports avg@$32$ on AIME 2024~\citep{aime24} / 2025~\citep{aime25} / 2026~\citep{aime26} and HMMT 2025~\citep{dekoninck2026matharena}, and avg@$4$ on MinervaMath~\citep{lewkowycz2022solving}. Full model list, sampling settings, gate-calibration details, and example teacher prompts are in Appendix~\ref{app:hparams} and~\ref{app:prompts}.

\subsection{Main results}
\label{sec:exp:main}
\vspace{-0.5em}

\begin{table}[ht]
  \caption{\textbf{Main results} (accuracy \%). AIME24/25/26 and HMMT25: avg@$32$; Minerva: avg@$4$. Subscript on \emph{Avg} = peak-mean step; \emph{Speedup} = GRPO's best-Avg step / this row's first-reach step ($\times$: never reached). \textbf{Bold} = column best within each model block; highlighted row is canonical AntiSD.}
  \vspace{-0.5em}
  \label{tab:main}
  \centering
  \small
  \setlength{\tabcolsep}{5pt}
  \renewcommand{\arraystretch}{0.95}
  \begin{tabular}{l|c c c c c|c|c}
    \toprule
    Method                             & \textbf{AIME24} & \textbf{AIME25} & \textbf{AIME26} & \textbf{HMMT25} & \textbf{Minerva} & \textbf{Average}                  & \textbf{Speedup}      \\
    \midrule
    \textit{Qwen3-8B}                  & 25.5            & 21.3            & 17.0            & 10.8            & 39.0             & 22.7                              & --                    \\
    \quad +GRPO                        & 73.5            & 65.2            & 64.2            & 39.2            & 45.1             & 57.4\textsubscript{@200}          & 1.0$\times$           \\
    \quad +SD                          & 40.1            & 30.5            & 26.9            & 14.9            & 40.7             & 30.6\textsubscript{@200}          & $\times$              \\
    \rowcolor{antisdrow} \quad +AntiSD & \textbf{78.4}   & \textbf{73.4}   & \textbf{73.7}   & \textbf{54.4}   & \textbf{48.5}    & \textbf{65.7\textsubscript{@180}} & \textbf{5.0$\times$}  \\
    \midrule
    \textit{Qwen3-4B-IT-2507}          & 62.1            & 48.2            & 53.8            & 30.4            & 43.4             & 47.6                              & --                    \\
    \quad +GRPO                        & 67.8            & 57.7            & 63.5            & 34.1            & 33.2             & 51.3\textsubscript{@200}          & 1.0$\times$           \\
    \quad +SD                          & 59.8            & 45.8            & 52.0            & 28.8            & 43.0             & 45.9\textsubscript{@10}           & $\times$              \\
    \rowcolor{antisdrow} \quad +AntiSD & \textbf{76.6}   & \textbf{70.2}   & \textbf{74.4}   & \textbf{46.7}   & \textbf{46.4}    & \textbf{62.8\textsubscript{@100}} & \textbf{10.0$\times$} \\
    \midrule
    \textit{Olmo3-7B-IT}               & 53.3            & 39.7            & 44.4            & 25.9            & 38.7             & 40.4                              & --                    \\
    \quad +GRPO                        & 57.0            & 45.3            & 52.1            & 31.2            & 29.1             & 43.0\textsubscript{@190}          & 1.0$\times$           \\
    \quad +SD                          & 54.5            & 41.8            & 46.6            & 24.4            & 38.5             & 41.1\textsubscript{@10}           & $\times$              \\
    \rowcolor{antisdrow} \quad +AntiSD & \textbf{62.4}   & \textbf{49.1}   & \textbf{55.2}   & \textbf{32.3}   & \textbf{42.4}    & \textbf{48.3\textsubscript{@200}} & \textbf{9.5$\times$}  \\
    \midrule
    \textit{Olmo3-7B-TK}               & 74.6            & 68.7            & 73.3            & 48.2            & 45.2             & 62.0                              & --                    \\
    \quad +GRPO                        & 76.5            & 71.7            & 75.3            & 50.5            & \textbf{46.4}    & 64.1\textsubscript{@80}           & 1.0$\times$           \\
    \quad +SD                          & 77.2            & 68.5            & 74.0            & 48.3            & 44.8             & 62.6\textsubscript{@40}           & $\times$              \\
    \rowcolor{antisdrow} \quad +AntiSD & \textbf{77.6}   & \textbf{75.3}   & \textbf{76.1}   & \textbf{56.2}   & 45.8             & \textbf{66.2\textsubscript{@40}}  & \textbf{2.0$\times$}  \\
    \midrule
    \textit{Qwen3-30B-A3B}             & 28.1            & 22.3            & 21.8            & 11.7            & 43.5             & 25.5                              & --                    \\
    \quad +GRPO                        & 74.1            & 66.8            & 65.5            & 42.2            & 47.1             & 59.1\textsubscript{@150}          & 1.0$\times$           \\
    \quad +SD                          & 40.9            & 32.5            & 34.0            & 20.9            & 44.1             & 34.5\textsubscript{@60}           & $\times$              \\
    \rowcolor{antisdrow} \quad +AntiSD & \textbf{79.4}   & \textbf{75.6}   & \textbf{74.1}   & \textbf{55.2}   & \textbf{49.7}    & \textbf{66.8\textsubscript{@140}} & \textbf{2.9$\times$}  \\
    \bottomrule
  \end{tabular}
\end{table}

Table~\ref{tab:main} reports avg@32 at each (model, method)'s best-Avg checkpoint. Three patterns hold:
\emph{(i) AntiSD reaches GRPO's accuracy in a fraction of the steps}, with a speedup of $2$--$10\times$ across all five models. The largest speedups appear on the smaller models with weaker GRPO baselines (Qwen3-4B-IT-2507 $10\times$, Olmo3-7B-IT $9.5\times$, Qwen3-8B $5\times$); the speedup shrinks but stays positive on the two strongest baselines (Olmo3-7B-TK $2\times$, where GRPO already sits at $64.1$; Qwen3-30B-A3B $2.9\times$, the $30$B mixture-of-experts model). This early ignition is consistent with the diagnosis in Section~\ref{sec:antisd:direction}: the per-token reward $-\varphi(u_t)$ is informative from the first step, so credit-assignment does not have to wait for sparse trajectory-level reward to propagate through the policy.
\emph{(ii) AntiSD's final mean accuracy exceeds GRPO's on every model}, by $+2.1$ to $+11.5$ points (Avg). The gap is widest on the weaker baselines ($+5.3$ to $+11.5$ on Qwen3-8B, Qwen3-4B-IT-2507, Olmo3-7B-IT), still substantial at scale ($+7.7$ on Qwen3-30B-A3B), and narrowest on the strongest GRPO baseline Olmo3-7B-TK ($+2.1$), where GRPO at $64.1$ and the un-trained base at $62.0$ leave little headroom on DAPO-Math-17k. Per-benchmark, $4$ of $5$ models win on every individual benchmark; the lone near-tie is a $-0.6$\,pp gap on MinervaMath for Olmo3-7B-TK, ruling out the explanation that one easy benchmark inflates the mean. The gain matches our prediction that biasing optimization toward deliberation tokens unlocks problems that GRPO's sparse signal cannot reach.
\emph{(iii) Default self-distillation underperforms the GRPO baseline} on every model, often by a wide margin (Qwen3-8B Avg: $30.6$ vs $57.4$). The mechanism behind this collapse, and the entropy dynamics that distinguish it from AntiSD, are examined in Section~\ref{sec:exp:dyn}.

\begin{figure}[ht]
  \centering
  \includegraphics[width=\linewidth]{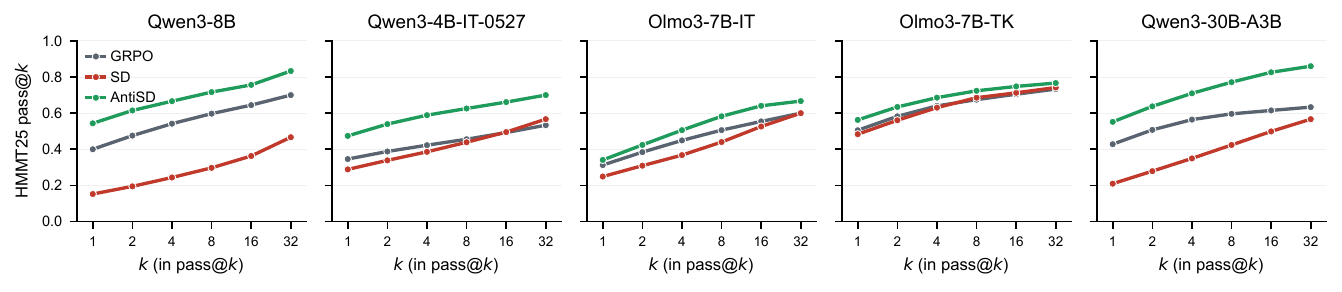}
  \vspace{-1em}
  \caption{{HMMT25 pass@$k$} at each row's peak-mean snapshot. AntiSD's lead over GRPO is sustained across $k$ -- the gain reflects expanded coverage, not just variance reduction.}
  \label{fig:pass_k}
  \vspace{-.5em}
\end{figure}

A natural concern is whether AntiSD's gain comes from better single-rollout accuracy or from concentrating probability mass on already-correct rollouts at the cost of generation diversity. Figure~\ref{fig:pass_k} plots pass@$k$ on HMMT 2025 (the hardest of the five benchmarks) to disentangle these. AntiSD's lead over GRPO is sustained across $k$: on Qwen3-8B the gap is $\sim 13$ points at $k=1$ and remains $\sim 7$--$10$ points at $k=32$. The non-converging curves at high $k$ indicate that AntiSD genuinely solves problems that GRPO cannot reach even with 32 attempts and preserves the rollout diversity needed to do so, rather than trading diversity for single-rollout consistency.

\begin{wraptable}{r}{0.40\linewidth}
  \vspace{-1.5em}
  \centering
  \small
  \setlength{\tabcolsep}{6pt}
  \renewcommand{\arraystretch}{0.95}
  \caption{{Code reasoning on Qwen3-8B} (avg@$10$). \textbf{Bold} marks the best per column.}
  \label{tab:code}
  \begin{tabular}{l|cc}
    \toprule
    \textbf{Method}              & \textbf{HumanEval+} & \textbf{MBPP+} \\
    \midrule
    +GRPO                        & 40.4                & 61.0           \\
    \rowcolor{antisdrow} +AntiSD & \textbf{41.6}       & \textbf{63.3}  \\
    \bottomrule
  \end{tabular}
  \vspace{-1.0em}
\end{wraptable}

\textbf{Code reasoning.} To probe whether AntiSD generalises beyond math, we run the same on-policy self-distillation setup on the Dolci-RLZero code RL dataset \citep{olmo2025olmo3} and evaluate on HumanEval+ and MBPP+~\citep{liu2023your} (Table~\ref{tab:code}). On Qwen3-8B, AntiSD improves over the GRPO baseline by $+1.2$ points on HumanEval+ and $+2.3$ on MBPP+; the gains are smaller than on math reasoning but consistent in direction, indicating that the per-token mechanism transfers to a setting where the trajectory-level reward is itself denser.

\subsection{Training dynamics}
\label{sec:exp:dyn}
\vspace{-0.5em}

Figure~\ref{fig:dyn} traces six training-time signals through the run. AntiSD ignites earliest: truncation-corrected train reward climbs from $\sim 0.5$ to $\sim 0.95$ within $\sim 30$ steps on Qwen3-8B and Qwen3-4B-IT-2507, a regime GRPO reaches only after $\sim 150$ steps and SD never reaches, with HMMT25 and AIME25 moving in lockstep. The Qwen3-4B-IT-2507 plateau sits near $0.95$ rather than $1.0$ and drifts slightly late in training; held-out accuracy does not drop, so this is saturation against the DAPO-Math problem distribution -- once almost every sampled problem is solved, the surviving gradient signal is noise -- rather than overfitting.

\begin{figure}[ht]
  \centering
  \includegraphics[width=\linewidth]{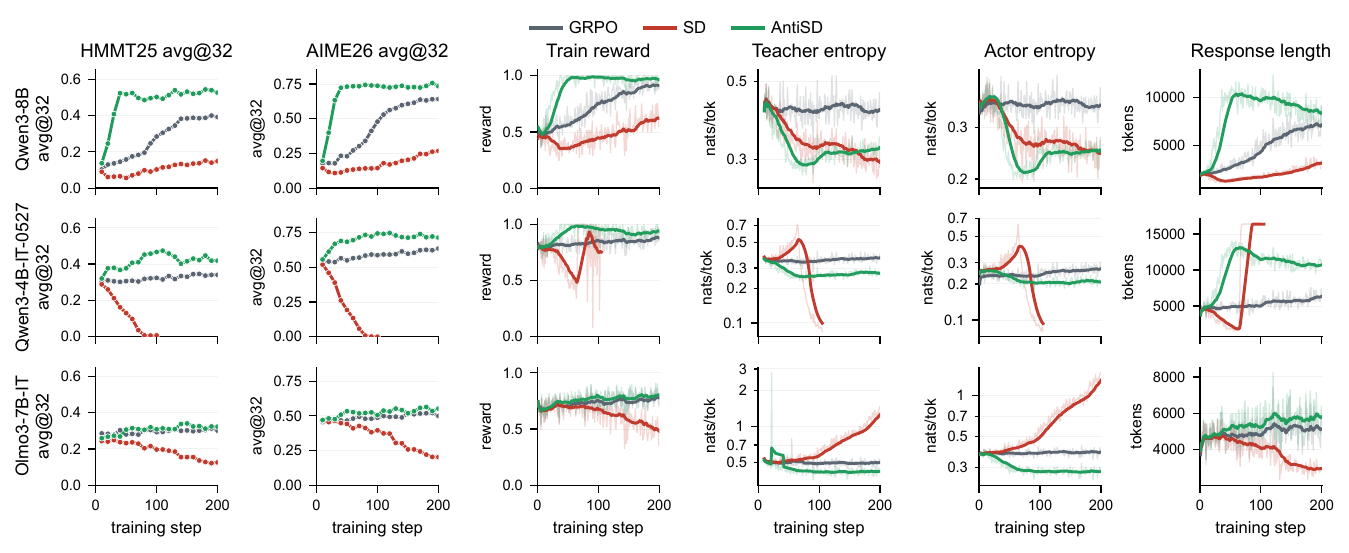}
  \caption{\textbf{Training dynamics} on three models (rows) along six axes (columns). Faded traces are raw values; bold traces are $20$-step rolling means.}
  \label{fig:dyn}
  \vspace{-0.5em}
\end{figure}

Default self-distillation diverges in opposite directions across model families. Both AntiSD and SD couple the student and teacher distributions, but their entropy traces tell different stories: AntiSD remains in a stable middle band on all three models, while SD's teacher and actor entropy collapse toward $\sim 0.1$ nats per token on Qwen3-4B-IT-2507 (over-confident on the shortcut answer template) and inflate past $1$ nat per token on Olmo3-7B-IT (drift away from useful tokens). This is exactly the bidirectional failure mode that the sign reversal in Section~\ref{sec:antisd:stab} addresses; the same shortcut bias that explains SD's gap to GRPO in Table~\ref{tab:main} is what is visibly amplifying or eroding teacher entropy here. Its sharpest expression is the Qwen3-4B-IT-2507 collapse around step $80$: train reward to zero, response length pinned at the $32$K cap, and both entropies spiking, all within a single step window before the run terminates.

\subsection{Ablations}
\label{sec:exp:ablation}
\vspace{-0.5em}

AntiSD adds three components on top of the GRPO advantage: sign-reversed reward $-\varphi(u_t)$, the JSD/softplus shape, and an entropy-triggered gate. Sign reversal is the dominant lever and was already established in Table~\ref{tab:main}: removing it (default SD) drops Qwen3-8B Avg from $65.7$ to $30.6$. We focus the remaining ablations on the other two components and on the privileged context itself, reporting both training-curve health (does the run survive?) and held-out accuracy on Qwen3-4B-IT-2507, the most failure-prone model in our suite.

\begin{figure}[ht]
  \centering
  \includegraphics[width=\linewidth]{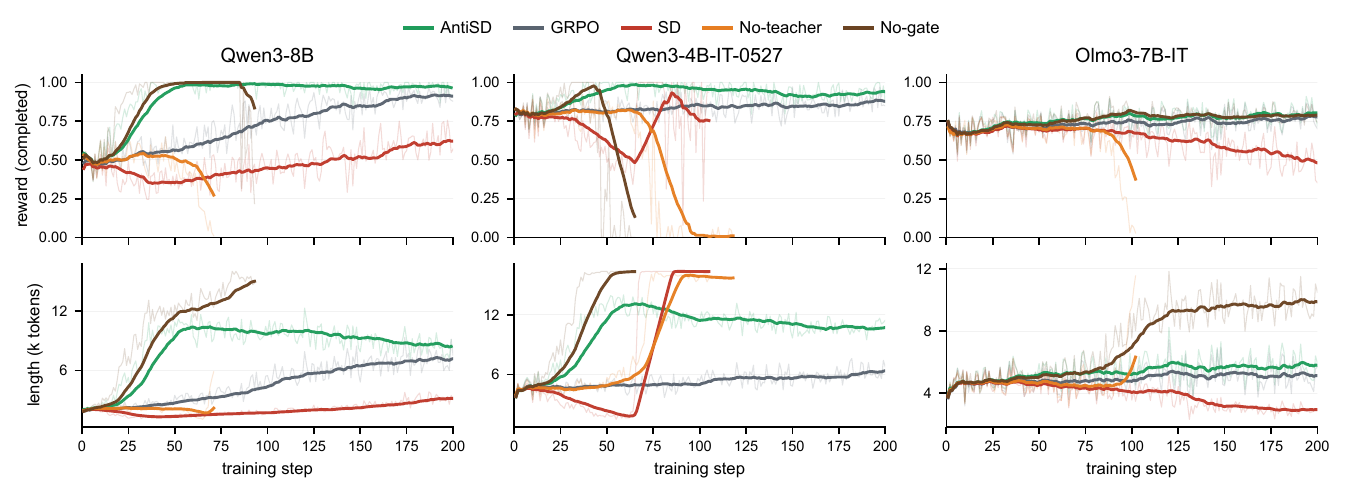}
  \caption{\textbf{Mechanism ablations: failure modes.} Top: reward over non-truncated rollouts. Bottom: mean response length. Line truncation indicates run termination after collapse.}
  \label{fig:fm}
  \vspace{-.5em}
\end{figure}

\textbf{No-teacher: self-reinforcement collapse.} Removing the teacher entirely -- so the per-token signal becomes a function of the student's log-probability alone, with no teacher--student differential -- collapses on all three models within $\sim 70$ training steps (Figure~\ref{fig:fm}, orange). Without external information from the privileged context, the per-token term degenerates into a function of the student's own probability, producing a positive-feedback signal that reinforces whatever the policy already emits; this is a textbook self-reinforcement collapse and is the strongest evidence in our suite that AntiSD's gain depends on the privileged information identity from Section~\ref{sec:antisd:direction}, not on a generic shaping of student log-probabilities. This contrasts with recent self-reward methods \citep{zhao2025learning,he2026far}, which keep an external signal (typically majority-vote agreement across rollouts) rather than removing all conditioning. Our No-teacher variant strips the privileged context entirely, leaving only $\pi_S$ in the loss, and that is precisely the configuration that fails to learn. AntiSD's privileged-context conditioning instead preserves rollout diversity (Figure~\ref{fig:pass_k}: sustained pass@$k$ lead over GRPO out to $k=32$), ruling out the mode-collapse-onto-majority failure that self-reward methods often incur.

\textbf{No-gate: stabilization is model-conditional.} Removing the entropy gate (always-on PRM) is the most striking model-dependent failure (Figure~\ref{fig:fm}, brown). On Qwen3-8B and Qwen3-4B-IT-2507 the No-gate run actually \emph{ignites faster} than canonical AntiSD -- reward peaks $\sim 0.97$ around step $40$ -- before collapsing near step $90$ as the teacher's per-token entropy crosses zero, at which point the PRM signal degenerates because the teacher has absorbed the answer template. On Olmo3-7B-IT the same configuration survives the full $200$ steps and even attains the highest plateau on this model. The asymmetry tracks initial teacher entropy: Qwen models start at $\approx 0.4$ nats per token (close to the absorption threshold) while Olmo starts higher, leaving headroom that the gate would have monitored. Read together with Section~\ref{sec:antisd:stab}, the gate acts as a cross-model insurance policy rather than a per-model necessity; it is essential for Qwen and inert for Olmo.

\begin{table}[ht]
  \caption{\textbf{Component sensitivity on Qwen3-4B-IT-2507.} Single-knob deviations from canonical AntiSD (highlighted). \emph{Speedup} as in Table~\ref{tab:main}. \textbf{Bold} = column best among +AntiSD rows.}
  \vspace{-0.5em}
  \label{tab:ablation}
  \centering
  \small
  \setlength{\tabcolsep}{1.5pt}
  \renewcommand{\arraystretch}{0.95}
  \begin{tabular}{l|ccc|cccccc|c}
    \toprule
    \textbf{Method}          & \textbf{Div} & $\boldsymbol{\tau_{\mathrm{down}}}$ & \textbf{Compose}             & \textbf{AIME24} & \textbf{AIME25} & \textbf{AIME26} & \textbf{HMMT25} & \textbf{Minerva} & \textbf{Average}                                & \textbf{Speedup}      \\
    \midrule
    GRPO                     & --           & --                                  & --                           & 67.8            & 57.7            & 63.5            & 34.1            & 33.2             & 51.3\textsubscript{@200}                        & 1.0$\times$           \\
    \midrule
    \multirow{7}{*}{+AntiSD} & rev.~KL      & $0.93$                              & add.                         & 64.1            & 49.0            & 58.9            & 32.0            & 43.7             & 49.5\textsubscript{@10}                         & $\times$              \\
    \cmidrule{2-11}
                             & JSD          & none                                & add.                         & 72.4            & 66.5            & 72.6            & 44.9            & 46.9             & 60.6\textsubscript{@30}                         & \textbf{10.0$\times$} \\
                             & JSD          & $0.90$                              & add.                         & 68.8            & 56.5            & 63.1            & 39.6            & 44.6             & 54.5\textsubscript{@20}                         & \textbf{10.0$\times$} \\
                             & JSD          & $0.95$                              & add.                         & \textbf{77.6}   & 69.4            & 63.4            & 44.7            & 46.2             & 60.3\textsubscript{@110}                        & 6.7$\times$           \\
    \cmidrule{2-11}
                             & JSD          & $0.93$                              & mult.                        & 70.8            & 61.5            & 67.6            & 36.6            & 46.1             & 56.5\textsubscript{@170}                        & 5.0$\times$           \\
                             & JSD          & $0.93$                              & add.\textsuperscript{$\ast$} & 74.8            & 65.3            & 70.1            & 39.4            & \textbf{51.8}    & 60.3\textsubscript{@20}\textsuperscript{$\ast$} & \textbf{10.0$\times$} \\
    \cmidrule{2-11}
    \rowcolor{antisdrow}{}   & {JSD}        & {0.93}                              & {add.}                       & 76.6            & \textbf{70.2}   & \textbf{74.4}   & \textbf{46.7}   & 46.4             & \textbf{62.8\textsubscript{@100}}               & \textbf{10.0$\times$} \\
    \bottomrule
  \end{tabular}
  \par\smallskip
  {\footnotesize $\ast$ Gate signal swapped from teacher- to student-perplexity.}
  \vspace{-.5em}
\end{table}

\textbf{Component sensitivity.} Table~\ref{tab:ablation} factorises the remaining design choices on Qwen3-4B-IT-2507 along three knobs: \emph{Div} (JSD vs reverse-KL ascent), the gate deactivation threshold $\tau_{\mathrm{down}}$, and \emph{Compose} (additive vs multiplicative shaping of the GRPO advantage). \emph{Threshold sensitivity is model-conditional.} Loosening $\tau_{\mathrm{down}}$ from $0.93$ to $0.90$ drops Q4 Avg by $8.3$ points ($62.8 \to 54.5$); on Qwen3-8B the same loosening leaves Avg essentially unchanged ($65.9$ vs canonical $65.7$; Appendix~\ref{app:ablation-qn}). The canonical $0.93$ is not a per-model sweet spot but the value that transfers across the models we evaluate. \emph{Additive composition outperforms multiplicative.} Replacing additive with multiplicative drops Avg by $6.3$ points ($62.8 \to 56.5$) and halves the speedup ($10\times \to 5\times$): when $A^{\mathrm{seq}}$ is small, the multiplicative form scales the AntiSD term down toward zero, removing the deliberation push exactly when GRPO's ORM signal is uninformative. The gate-off row's $60.6$ peak at step $30$ is a transient pre-collapse value (Figure~\ref{fig:fm}, brown), not a sustainable plateau.

\subsection{Beyond GRPO saturation}
\label{sec:exp:continual}
\vspace{-0.5em}

A practical question is whether AntiSD must be trained from the base model, or whether the advantage shaping can be applied as a refinement on top of an existing GRPO checkpoint. We resume the Qwen3-8B GRPO run at step $200$ -- the saturation point in Table~\ref{tab:main} -- and continue with the canonical AntiSD configuration for $50$ further steps. Optimizer state, dataloader index, and reference policy are inherited from the GRPO run via the standard verl~\citep{sheng2025hybridflow} resume path; the gate threshold is recalibrated against the new $H_{\mathrm{warm}}$ to account for the shifted teacher-entropy distribution of the saturated policy.

\begin{table}[ht]
  \caption{\textbf{Continual AntiSD on Qwen3-8B.} \emph{+AntiSD$^{\dag}$} resumes from GRPO@$200$; \emph{Steps} counts post-resume only. \emph{Speedup} as in Table~\ref{tab:main}. \textbf{Bold} = column best.}
  \vspace{-0.5em}
  \label{tab:continual}
  \centering
  \small
  \setlength{\tabcolsep}{5pt}
  \renewcommand{\arraystretch}{0.95}
  \begin{tabular}{l|c|ccccc|c|c}
    \toprule
    \textbf{Method}                       & \textbf{Steps} & \textbf{AIME24} & \textbf{AIME25} & \textbf{AIME26} & \textbf{HMMT25} & \textbf{Minerva} & \textbf{Average}                  & \textbf{Speedup}     \\
    \midrule
    GRPO                                  & 200            & 73.5            & 65.2            & 64.2            & 39.2            & 45.1             & 57.4\textsubscript{@200}          & 1.0$\times$          \\
    +AntiSD                               & 200            & \textbf{78.4}   & \textbf{73.4}   & \textbf{73.7}   & \textbf{54.4}   & 48.5             & \textbf{65.7\textsubscript{@180}} & \textbf{5.0$\times$} \\
    \rowcolor{antisdrow} +AntiSD$^{\dag}$ & +50            & 74.9            & 72.4            & 73.0            & 54.0            & \textbf{50.6}    & 65.0\textsubscript{@30}           & \textbf{5.0$\times$} \\
    \bottomrule
  \end{tabular}
  \vspace{-.5em}
\end{table}

Table~\ref{tab:continual} contrasts three policies: the saturated GRPO baseline at step $200$, AntiSD trained from base, and continual AntiSD that initialises from the GRPO checkpoint. Continual AntiSD essentially matches the from-base peak ($65.0$ vs $65.7$ Avg) using only $30$ post-resume steps -- $6\times$ fewer than the $180$ steps from-base AntiSD takes to reach its own peak -- demonstrating that AntiSD's advantage stacks on top of GRPO rather than replacing it: the deliberation-token reward signal remains informative even at the saturation point that GRPO's trajectory-level reward cannot push past. The same experiment on Qwen3-4B-IT-2507 (Appendix~\ref{app:continual-q4}) comes within $\sim 1$\,pp of the from-base peak before settling slightly lower, suggesting GRPO's basin admits most but not all of AntiSD's deliberation pressure.

\vspace{-0.5em}

\section{Related Work}
\label{sec:related}
\vspace{-0.5em}

\textbf{On-policy self-distillation.}
A series of recent works develops on-policy self-distillation along parallel axes \citep{opsd, sdpo, opcd-ye, opsdc, yang2026rlsd}, all using a teacher--student log-ratio in which the teacher is the student conditioned on privileged context (a verified solution and any environment feedback). These methods build on on-policy distillation \citep{gkd, fu2026revisitingonpolicydistillationempirical} (student-sampled rollouts with an external teacher) and learning under privileged information \citep{vapnik-lupi, lopez2015unifying}, and share the same gradient direction. AntiSD inverts that direction from the start of training, replacing descent on reverse KL with bounded ascent on Jensen--Shannon divergence; an auto-calibrated entropy-triggered gate then disables the term once the teacher's per-token entropy collapses.

\textbf{Diagnoses of self-distillation.}
Self-distillation has been diagnosed as degrading reasoning capability both directly \citep{kim2026does} and implicitly via its framing as a response-compression tool \citep{opsdc}. Existing self-distillation methods \citep{opsd, sdpo, opcd-ye, yang2026rlsd} report mainly on simpler benchmarks, not the AIME / HMMT-class problems where we observe the clearest failure. A broader line of on-policy distillation diagnostics \citep{fu2026revisitingonpolicydistillationempirical, li2026rethinking, li2026unifying, xu2026paceddistillationselfdistillationfrontier} documents teacher--student capability gaps and distribution mismatch without isolating self-distillation from external-teacher OPD. We confirm the same symptom across model families from 4B to 30B parameters (Section~\ref{sec:exp}), but trace it to a structural property of the per-token signal itself: under privileged context it is conditional pointwise mutual information between the next token and that context (Section~\ref{sec:antisd:direction}), which biases credit toward tokens the context already implies and away from deliberation tokens. AntiSD acts on this mechanism by reversing the per-token sign rather than by reweighting samples or filtering teacher confidence.

\textbf{Process reward models and reward shaping.}
To address sparse credit assignment in RLVR, a separate line of work trains process reward models that score intermediate reasoning steps \citep{prm800k, math-shepherd, luo2024improve, setlur2025rewarding, lee2026efficient}, either from human annotations or from Monte Carlo rollout estimates of step value, while implicit-reward methods such as PRIME \citep{cui2025process} derive process rewards from preference signals jointly with the policy. AntiSD's per-token signal is structurally a PRM, but a training-free one: it is the difference $V_t - V_{t-1}$ of the model's own log-posterior $V_t = \log P(c \mid x, y_{\le t})$ for the privileged context~$c$. This places the signal in the framework of potential-based reward shaping \citep{ng1999policy}: the per-token contributions telescope over a trajectory to the trajectory-level pointwise mutual information $\log P(c\mid x,y) - \log P(c\mid x)$, so the shaping term leaves the set of optimal policies invariant. The PRM is therefore obtained without auxiliary annotation, learned step-value heads, or Monte Carlo rollouts. Our contribution is to identify the conditional PMI bias of this shaping signal (inflating shortcut tokens, suppressing deliberation tokens) and correct it through gradient-direction inversion rather than as an additional reward channel.

\vspace{-0.5em}

\section{Conclusion}
\label{sec:conclusion}
\vspace{-0.5em}

We identified the per-token signal of default on-policy self-distillation as conditional pointwise mutual information between the next token and the privileged context, exposing a structural shortcut bias that rewards tokens the context already implies and penalises the deliberation tokens that drive search. Anti-Self-Distillation responds by inverting the gradient direction from the first training step, replacing reverse-KL descent with bounded Jensen--Shannon ascent; a single auto-calibrated, entropy-triggered gate then disables the term once the teacher's per-token entropy collapses, preventing the run-away drift that pure ascent would otherwise incur. Across five language models from $4$B to $30$B parameters, AntiSD reaches a strong GRPO baseline's accuracy in $2$ to $10\times$ fewer training steps and improves final accuracy by up to $11.5$ points; the gap between AntiSD and default self-distillation in Table~\ref{tab:main} is consistent with sign reversal carrying most of the gain on top of an already-bounded shape. The PMI characterisation describes individual gradient contributions rather than the global optimum of the combined objective, and our evaluation focuses on math reasoning, leaving extensions to multi-turn agentic settings and broader coding benchmarks as natural next directions. A fuller discussion of limitations and broader impacts is in Appendix~\ref{app:impacts}.

\bibliographystyle{plainnat}
\bibliography{refs}

\newpage
\appendix

\begin{center}
  \noindent\rule{\linewidth}{3pt}\\[6pt]
  {\Large\bfseries Anti-Self-Distillation for Reasoning RL via\\[0.25em] Pointwise Mutual Information\\[0.4em] Supplementary Material}\\[6pt]
  \noindent\rule{\linewidth}{1pt}
\end{center}
\vspace{1em}

\noindent\textbf{Table of Contents}
\vspace{0.3em}
\hrule
\vspace{0.5em}
\noindent
\hyperref[app:proofs]{\ref{app:proofs}}\hspace{1em}Proofs and Deferred Statements\dotfill\pageref{app:proofs}\\[0.3em]
\hyperref[app:hparams]{\ref{app:hparams}}\hspace{1em}Hyperparameters\dotfill\pageref{app:hparams}\\[0.3em]
\hyperref[app:prompts]{\ref{app:prompts}}\hspace{1em}Self-Teacher Context Examples\dotfill\pageref{app:prompts}\\[0.3em]
\hyperref[app:experiments]{\ref{app:experiments}}\hspace{1em}Additional Experiments\dotfill\pageref{app:experiments}\\[0.3em]
\hspace*{1.5em}\hyperref[app:ablation-qn]{\ref{app:ablation-qn}}\hspace{1em}Component sensitivity on Qwen3-8B\dotfill\pageref{app:ablation-qn}\\[0.3em]
\hspace*{1.5em}\hyperref[app:continual-q4]{\ref{app:continual-q4}}\hspace{1em}Continual AntiSD on Qwen3-4B-IT-2507\dotfill\pageref{app:continual-q4}\\[0.3em]
\hyperref[app:impacts]{\ref{app:impacts}}\hspace{1em}Limitations and Broader Impacts\dotfill\pageref{app:impacts}
\vspace{0.5em}
\hrule
\vspace{1.5em}

\section{Proofs and Deferred Statements}
\label{app:proofs}

This appendix collects the derivations summarised in Section~\ref{sec:antisd}: the reverse-KL gradient identity (Equation~\eqref{eq:thm1-grad}, Lemma~\ref{lem:rkl-grad}); the PMI characterization of $u_t$ under self-distillation (Equation~\eqref{eq:pmi}, Lemma~\ref{lem:pmi}) and its trajectory-level telescope into a potential-based shaping term (Lemma~\ref{lem:telescope}); the JSD f-divergence shape used in Equation~\eqref{eq:antisd-advantage} (Lemma~\ref{lem:jsd-shape}); and the properties of $\varphi$ relied on in Section~\ref{sec:antisd:stab} (Lemma~\ref{lem:phi-props}). All distributions are over the next-token vocabulary; we suppress the conditioning $(x, y_{<t})$ where it is fixed.

\begin{lemma}[Reverse-KL gradient identity, Equation~\eqref{eq:thm1-grad}]
  \label{lem:rkl-grad}
  Let $\pi_S(\cdot) = \pi_\theta(\cdot \mid x, y_{<t})$ and $\pi_T(\cdot) = \pi_\theta(\cdot \mid x, c, y_{<t})$ with stop-gradient on $\pi_T$, and write $u_v := \log \pi_T(v) - \log \pi_S(v)$. Then
  \begin{equation*}
    \nabla_\theta\, D_{\mathrm{KL}}(\pi_S \| \pi_T)
    \;=\; -\,\mathbb{E}_{v \sim \pi_S}\!\left[ u_v \cdot \nabla_\theta \log \pi_S(v) \right].
  \end{equation*}
\end{lemma}

\begin{proof}
  Expanding $D := D_{\mathrm{KL}}(\pi_S \| \pi_T) = \sum_v \pi_S(v)\,(\log \pi_S(v) - \log \pi_T(v))$ and applying the product rule,
  \begin{equation*}
    \nabla_\theta D
    \;=\;
    \underbrace{\sum_v \nabla_\theta \pi_S(v)\,(\log \pi_S(v) - \log \pi_T(v))}_{\text{(I)}}
    \;+\;
    \underbrace{\sum_v \pi_S(v)\,\nabla_\theta (\log \pi_S(v) - \log \pi_T(v))}_{\text{(II)}}.
  \end{equation*}
  Term (II) vanishes: $\nabla_\theta \log \pi_T(v) = 0$ by stop-gradient, and $\sum_v \pi_S(v)\,\nabla_\theta \log \pi_S(v) = \sum_v \nabla_\theta \pi_S(v) = \nabla_\theta \sum_v \pi_S(v) = \nabla_\theta 1 = 0$. For term (I), the score-function identity $\nabla_\theta \pi_S(v) = \pi_S(v)\,\nabla_\theta \log \pi_S(v)$ gives
  \begin{equation*}
    \text{(I)} \;=\; \mathbb{E}_{v \sim \pi_S}\!\left[(\log \pi_S(v) - \log \pi_T(v))\,\nabla_\theta \log \pi_S(v)\right] \;=\; -\,\mathbb{E}_{v \sim \pi_S}\!\left[u_v \nabla_\theta \log \pi_S(v)\right],
  \end{equation*}
  where $\log \pi_S - \log \pi_T = -u$. Combining $\text{(I)} + \text{(II)}$ proves the claim.
\end{proof}

\begin{lemma}[PMI characterization, Equation~\eqref{eq:pmi}]
  \label{lem:pmi}
  Under the self-distillation parameter sharing $\pi_S(\cdot) = \pi_\theta(\cdot \mid x, y_{<t})$ and $\pi_T(\cdot) = \pi_\theta(\cdot \mid x, c, y_{<t})$,
  \begin{equation*}
    u_t \;=\; \log \frac{\pi_\theta(y_t \mid x, c, y_{<t})}{\pi_\theta(y_t \mid x, y_{<t})} \;=\; \mathrm{PMI}(y_t\,;\, c \mid x, y_{<t}).
  \end{equation*}
\end{lemma}

\begin{proof}
  Bayes' rule applied to the joint of $(y_t, c)$ given $(x, y_{<t})$ gives
  \begin{equation*}
    \pi_\theta(y_t \mid x, c, y_{<t}) \;=\; \frac{\pi_\theta(c \mid x, y_{\le t})\,\pi_\theta(y_t \mid x, y_{<t})}{\pi_\theta(c \mid x, y_{<t})},
  \end{equation*}
  so $\frac{\pi_\theta(y_t \mid x, c, y_{<t})}{\pi_\theta(y_t \mid x, y_{<t})} = \frac{\pi_\theta(c \mid x, y_{\le t})}{\pi_\theta(c \mid x, y_{<t})}$, and taking logs yields the conditional pointwise mutual information $\mathrm{PMI}(y_t; c \mid x, y_{<t})$.
\end{proof}

\begin{lemma}[Trajectory-level potential shaping]
  \label{lem:telescope}
  Summing $u_t$ over a complete trajectory telescopes to a sequence-level pointwise mutual information:
  \begin{equation*}
    \sum_{t=1}^{T} u_t \;=\; \log \pi_\theta(c \mid x, y) - \log \pi_\theta(c \mid x) \;=\; \mathrm{PMI}(y\,;\, c \mid x).
  \end{equation*}
  Hence the per-token contributions $\{u_t\}$ are the increments of a potential $\Phi_t := \log \pi_\theta(c \mid x, y_{\le t})$, and the augmented advantage in Equation~\eqref{eq:adv-combined} is a potential-based reward shaping in the sense of \citet{ng1999policy}: it leaves the set of optimal policies invariant for any underlying scalar reward.
\end{lemma}

\begin{proof}
  By Lemma~\ref{lem:pmi}, $u_t = \log \pi_\theta(c \mid x, y_{\le t}) - \log \pi_\theta(c \mid x, y_{<t}) = \Phi_t - \Phi_{t-1}$. The sum telescopes to $\Phi_T - \Phi_0 = \log \pi_\theta(c \mid x, y) - \log \pi_\theta(c \mid x)$. The potential-based shaping invariance result follows directly~\citep{ng1999policy}.
\end{proof}

\begin{lemma}[JSD f-divergence shape, Equation~\eqref{eq:antisd-advantage}]
  \label{lem:jsd-shape}
  Write the symmetric Jensen--Shannon divergence in f-divergence form $D_{\mathrm{JSD}}(\pi_S \| \pi_T) = \mathbb{E}_{\pi_T}[f(\pi_S / \pi_T)]$ with generator $f(r) = \tfrac12 r \log \frac{2r}{1+r} + \tfrac12 \log \frac{2}{1+r}$. Then $f'(r) = \tfrac{1}{2} \log \frac{2r}{1+r}$ and, for $r = \pi_S/\pi_T = e^{-u}$,
  \begin{equation*}
    f'\!\left(\tfrac{\pi_S}{\pi_T}\right) \;=\; \tfrac{1}{2}\!\left(\log 2 - \mathrm{softplus}(u)\right) \;=\; -\varphi(u),
  \end{equation*}
  recovering the AntiSD advantage $A_t^{\mathrm{AntiSD}} = -\varphi(u_t)$ via the score-function identity in Equation~\eqref{eq:fdiv-grad}.
\end{lemma}

\begin{proof}
  Differentiating $f$ term by term, with $g(r) := \log\frac{2r}{1+r} = \log 2 + \log r - \log(1+r)$ and $g'(r) = \frac{1}{r} - \frac{1}{1+r} = \frac{1}{r(1+r)}$, the product rule gives
  \begin{equation*}
    f'(r) = \tfrac12 g(r) + \tfrac12 r\,g'(r) - \tfrac12 \cdot \tfrac{1}{1+r} = \tfrac12\!\left[g(r) + \tfrac{1}{1+r} - \tfrac{1}{1+r}\right] = \tfrac12 \log \tfrac{2r}{1+r}.
  \end{equation*}
  Substituting $r = e^{-u}$ gives $\frac{2r}{1+r} = \frac{2}{e^u+1}$, so $f'(e^{-u}) = \tfrac12(\log 2 - \log(1+e^u)) = \tfrac12(\log 2 - \mathrm{softplus}(u)) = -\varphi(u)$.
\end{proof}

\begin{lemma}[Properties of $\varphi$]
  \label{lem:phi-props}
  The shape $\varphi(u) := \tfrac12(\mathrm{softplus}(u) - \log 2)$ from Equation~\eqref{eq:antisd-advantage} satisfies:
  \begin{enumerate}
    \item[\textnormal{(i)}] \emph{Strict monotonicity:} $\varphi'(u) = \tfrac{1}{2}\sigma(u) > 0$, where $\sigma(u) = (1+e^{-u})^{-1}$ is the logistic function.
    \item[\textnormal{(ii)}] \emph{Sign preservation:} $\varphi(0) = 0$ and $\mathrm{sign}(\varphi(u)) = \mathrm{sign}(u)$.
    \item[\textnormal{(iii)}] \emph{One-sided bound:} $\varphi(u) \ge -\tfrac{1}{2}\log 2$ for all $u \in \mathbb{R}$, with equality attained as $u \to -\infty$; $\varphi$ is unbounded above.
  \end{enumerate}
\end{lemma}

\begin{proof}
  \textnormal{(i)} $\frac{d}{du}\mathrm{softplus}(u) = \frac{e^u}{1+e^u} = \sigma(u)$, so $\varphi'(u) = \tfrac12 \sigma(u) > 0$ for all $u$.
  \textnormal{(ii)} $\mathrm{softplus}(0) = \log 2$, so $\varphi(0) = 0$; combined with (i), $\varphi$ is strictly increasing through $0$.
  \textnormal{(iii)} $\mathrm{softplus}(u) = \log(1 + e^u) > 0$ for all $u$, with $\inf_u \mathrm{softplus}(u) = 0$ attained as $u \to -\infty$, hence $\varphi(u) \ge -\tfrac{1}{2}\log 2$ with equality in the limit. As $u \to \infty$, $\mathrm{softplus}(u) \sim u \to \infty$, so $\varphi$ has no upper bound.
\end{proof}

Properties (i)--(ii) ensure the shape inherits the per-token sign structure of Equations~\eqref{eq:thm1-grad}--\eqref{eq:pmi}: $-\varphi$ flips that sign at the source, so deliberation tokens ($u_t < 0$) receive positive advantage and shortcut tokens ($u_t > 0$) receive negative advantage. Property (iii) bounds the AntiSD advantage at $\tfrac12 \log 2$ on the deliberation side -- the side that observation (O2) flagged as both over-sampled and heavy-tailed -- so the cap re-balances per-token gradient contributions against the lighter shortcut side. The entropy gate then disables the term entirely once $u_t$ degenerates into floor-level noise.

\section{Hyperparameters}
\label{app:hparams}

\begin{algorithm}[h]
  \caption{Anti-Self-Distillation (AntiSD) -- one training step.}
  \label{alg:antisd}
  \begin{algorithmic}[1]
    \Require Policy $\pi_\theta$, batch $\{(x_i, y_i, c_i)\}_{i=1}^B$ with sequence-level GRPO advantage $A_i^{\mathrm{seq}}$; hyperparameter $\lambda_{\max}$; gate state $g$ and calibrated threshold $\tau_{\mathrm{down}}$ (warmup baseline $H_{\mathrm{warm}}$).
    \For{each rollout $i$ and token $t \in \{1,\ldots,T_i\}$}
    \State $s_{i,t} \gets \log \pi_\theta(y_{i,t} \mid x_i, y_{i,<t})$ \Comment{student log-prob}
    \State $t_{i,t} \gets \mathrm{stopgrad}\!\left(\log \pi_\theta(y_{i,t} \mid x_i, c_i, y_{i,<t})\right)$ \Comment{teacher log-prob}
    \State $u_{i,t} \gets t_{i,t} - s_{i,t}$ \label{algline:u} \Comment{$= \mathrm{PMI}(y_{i,t};\, c_i \mid x_i, y_{i,<t})$, see Equation~\eqref{eq:pmi}}
    \State $\varphi_{i,t} \gets \tfrac{1}{2}(\mathrm{softplus}(u_{i,t}) - \log 2)$ \label{algline:phi} \Comment{JSD f-divergence advantage; see Equation~\eqref{eq:antisd-advantage}}
    \EndFor
    \State $H \gets \mathrm{median}_{i,t}\, H[\pi_T(\cdot \mid x_i, y_{i,<t})]$ \Comment{teacher entropy}
    \State Update gate: $g \gets 1$ if $H \ge H_{\mathrm{warm}}$, $g \gets 0$ if $H < \tau_{\mathrm{down}}$, else unchanged.
    \State $\lambda \gets g \cdot \lambda_{\max}$ \label{algline:lambda}
    \State $A_{i,t} \gets A_i^{\mathrm{seq}} - \lambda \cdot \mathrm{stopgrad}(\varphi_{i,t})$ \label{algline:advantage} \Comment{ascent on $D_{\mathrm{JSD}}(\pi_S \| \pi_T)$; advantage is treated as a constant weight (cf.\ Eq.~\eqref{eq:fdiv-grad})}
    \State Update $\theta$ via standard policy gradient using $\{A_{i,t}\}$.
  \end{algorithmic}
\end{algorithm}

We evaluate five language models: \textbf{Qwen3-8B}, \textbf{Qwen3-4B-Instruct-2507}, \textbf{Olmo-3-7B-Instruct}, \textbf{Olmo-3-7B-Think}, and \textbf{Qwen3-30B-A3B}. All models share the configuration below, with training at $32$K maximum sequence length; the only model-conditional knob is the evaluation-time maximum sequence length, which is doubled for the thinking-model variant (Olmo-3-7B-Think) to accommodate longer chains-of-thought. AntiSD adds the gate parameters in the bottom block; the auto-calibration of $H_{\mathrm{warm}}$ and $\tau_{\mathrm{down}}$ is described in Section~\ref{sec:exp} (Setup).

\begin{table}[h]
  \centering
  \caption{Training and evaluation hyperparameters. AntiSD shares the GRPO/SD configuration and adds only the gate parameters in the bottom block.}
  \label{tab:hparams}
  \setlength{\tabcolsep}{4pt}
  \small
  \begin{tabular}{ll}
    \toprule
    \textbf{Hyperparameter}                              & \textbf{Value}                                            \\
    \midrule
    Optimizer                                            & AdamW                                                     \\
    Learning rate                                        & $1 \times 10^{-6}$                                        \\
    Training steps                                       & 200                                                       \\
    Batch size (problems)                                & 32                                                        \\
    Rollouts per problem (group size)                    & 8                                                         \\
    Max sequence length (training)                       & 32K (all models)                                          \\
    Max sequence length (evaluation)                     & 32K (Qwen, OLMo-Instruct), 64K (OLMo-Think)               \\
    GRPO clip ratio                                      & 0.2                                                       \\
    KL penalty coefficient                               & $0$ (no reference-policy KL term, following~\citep{dapo}) \\
    Reference model $\pi_{\text{ref}}$                   & frozen base model                                         \\
    Training rollout sampling                            & temperature $1.0$, top-$p\!=\!1.0$                        \\
    Evaluation rollout sampling                          & temperature $0.7$, top-$p\!=\!0.95$                       \\
    Evaluation rollouts per problem                      & 32 (AIME, HMMT) / 4 (MinervaMath)                         \\
    Hardware                                             & 8 NVIDIA H20 GPUs per node (multi-node for Qwen3-30B-A3B) \\
    Framework                                            & verl~\citep{sheng2025hybridflow}                          \\
    \midrule
    AntiSD mixing weight $\lambda_{\max}$                & 0.5                                                       \\
    AntiSD warmup length $W$                             & 5 steps                                                   \\
    AntiSD deactivation threshold $\tau_{\mathrm{down}}$ & $0.93 \cdot H_{\mathrm{warm}}$                            \\
    AntiSD reactivation threshold                        & $H_{\mathrm{warm}}$                                       \\
    \bottomrule
  \end{tabular}
\end{table}

\section{Self-Teacher Context Examples}
\label{app:prompts}

\definecolor{sysblue}{HTML}{2B5797}
\definecolor{respgreen}{HTML}{1A7A4C}
\definecolor{fbredcolor}{HTML}{B91C1C}

We show the context that the self-teacher $\pi_T(\cdot \mid x, y_{<t}, c)$ sees when re-evaluating the student's response. The teacher's input concatenates the original prompt, a verified solution (sampled from a successful rollout in the same batch when available, else from the dataset), and a binary feedback string indicating correctness, followed by an instruction to re-solve the problem. The student's original response $y$ is placed in the assistant role; the teacher then re-evaluates $y$'s log-probabilities under this enriched context. Templates follow~\citep{sdpo,opsd} and are identical for math and code; only the verified-solution and feedback strings differ by task. Colors: {\color{sysblue}original prompt}, {\color{fbredcolor}feedback string} (a sub-component of $c$), {\color{respgreen}student response $y$ (re-evaluated by teacher)}.

\paragraph{Reading the template.} The block under \emph{Your previous attempt:} carries the verified solution (a peer rollout or a dataset reference); the \emph{Previous assessment:} line is the binary correctness feedback for the student's actual rollout $y$ that follows in the assistant turn, not for the verified solution shown above. The two slots therefore play complementary roles: the verified solution narrows the teacher's posterior toward the correct answer, while the assessment string indicates that the trajectory the teacher will now re-evaluate is the student's own (which may be wrong). The deliberate asymmetry between ``correct reference'' and ``incorrect attempt'' gives the teacher a contrastive signal, and matches the prompt structure used in prior on-policy self-distillation work~\citep{sdpo,opsd}.

\begin{tcolorbox}[breakable,colback=gray!3,colframe=gray!50,title={\small Math (AIME / HMMT / MinervaMath) --- self-teacher context},fonttitle=\bfseries\small,boxrule=0.5pt,arc=2pt,left=4pt,right=4pt,top=2pt,bottom=2pt]
  \small\ttfamily
  {\color{sysblue}\textbf{[User]}} {\color{sysblue}Solve the following math problem. Place the final answer in $\backslash$boxed\{\}.}\\
  {\color{sysblue}Find the number of integer pairs $(a,b)$ with $1 \le a,b \le 100$ such that $\gcd(a,b)+\mathrm{lcm}(a,b)=a+b+50$.}\\[3pt]
  Your previous attempt:\\
  \textrm{\textit{(a successful rollout from the same batch, or a reference solution from the dataset)}}\\
  Let $d=\gcd(a,b)$, $a=dm$, $b=dn$\ldots so $d(1+mn)=dm+dn+50$\ldots\\
  $\backslash$boxed\{$25$\}\\[3pt]
  {\color{fbredcolor}Previous assessment: Your answer is incorrect.}\\[3pt]
  Now solve this problem step by step.\\[3pt]
  {\color{respgreen}\textbf{[Assistant]}} {\color{respgreen}\textit{(student's original response $y$, re-evaluated by teacher)}}\\
  {\color{respgreen}We start by writing $a=dm$, $b=dn$ with $\gcd(m,n)=1$\ldots}\\
  {\color{respgreen}$\backslash$boxed\{$30$\}}
\end{tcolorbox}

\begin{tcolorbox}[breakable,colback=gray!3,colframe=gray!50,title={\small Code (LiveCodeBench v6 / Dolci-RLZero) --- self-teacher context},fonttitle=\bfseries\small,boxrule=0.5pt,arc=2pt,left=4pt,right=4pt,top=2pt,bottom=2pt]
  \small\ttfamily
  {\color{sysblue}\textbf{[User]}} {\color{sysblue}You are a coding expert. Write a correct Python function that solves the following problem; place the final code in a $\backslash$texttt\{```python\} block.}\\
  {\color{sysblue}Given a list of integers \texttt{nums}, partition it into the minimum number of strictly increasing subsequences and return that number.}\\[3pt]
  Your previous attempt:\\
  \textrm{\textit{(a successful rollout from the same batch, or a reference implementation)}}\\
  \texttt{from bisect import bisect\_left}\\
  \texttt{def partitions(nums): tails = []; for x in nums: \ldots; return len(tails)}\\[3pt]
  {\color{fbredcolor}Previous assessment: Your code passes 7 of 12 test cases.}\\[3pt]
  Now solve this problem step by step.\\[3pt]
  {\color{respgreen}\textbf{[Assistant]}} {\color{respgreen}\textit{(student's original response $y$, re-evaluated by teacher)}}\\
  {\color{respgreen}We track the tail of each subsequence in sorted order\ldots}\\
  {\color{respgreen}\texttt{def partitions(nums): \ldots}}
\end{tcolorbox}

The feedback string is the only task-conditional piece. For math we use a binary form (``Your answer is correct.'' / ``Your answer is incorrect.''); for code, we use a continuous form (``Your code passes $N$ of $M$ test cases.'') matching the per-test fraction returned by the reward function. Both forms parallel the task's underlying score: math is exact-match boolean, while code's score is a fraction over executed test cases.

\section{Additional Experiments}
\label{app:experiments}

This section gathers experimental results referenced from the main paper but deferred to the appendix for space. Each subsection corresponds to one of the experimental probes whose narrative is summarised in Section~\ref{sec:exp}.

\subsection{Component sensitivity on Qwen3-8B}
\label{app:ablation-qn}

Table~\ref{tab:ablation-qn} mirrors the Qwen3-4B-IT-2507 ablation in Section~\ref{sec:exp:ablation} on Qwen3-8B. Two patterns from the main analysis carry over: \emph{rev.\ KL ascent collapses} ($30.6$ Avg, $-35.1$ pp from canonical) and \emph{the gate is necessary on this model family} (the no-gate run shows a transient peak at step $\sim 40$ before collapsing by step $\sim 90$, echoing the dynamics in Section~\ref{sec:exp:dyn}). The threshold-sensitivity story differs in direction: loosening from $0.93$ to $0.90$ slightly \emph{improves} Avg on Qwen3-8B ($65.7 \to 65.9$), in stark contrast to the $-8.3$\,pp drop on Qwen3-4B-IT-2507. Tightening to $0.95$ slightly drops the peak ($65.7 \to 65.4$) and slows ignition by $\sim 4\times$. The canonical $0.93$ is therefore not the per-model optimum on Qwen3-8B but the value that transfers across all models we evaluate without per-model retuning.

\begin{table}[t]
  \caption{\textbf{Component sensitivity on Qwen3-8B.} Same format as Table~\ref{tab:ablation}. \emph{Speedup} uses GRPO's best-Avg step ($200$). \textbf{Bold} = column best among +AntiSD rows.}
  \label{tab:ablation-qn}
  \centering
  \small
  \setlength{\tabcolsep}{1.5pt}
  \renewcommand{\arraystretch}{1.0}
  \begin{tabular}{l|ccc|cccccc|c}
    \toprule
    \textbf{Method}          & \textbf{Div} & $\boldsymbol{\tau_{\mathrm{down}}}$ & \textbf{Compose}             & \textbf{AIME24} & \textbf{AIME25} & \textbf{AIME26} & \textbf{HMMT25} & \textbf{Minerva} & \textbf{Average}                                   & \textbf{Speedup}     \\
    \midrule
    GRPO                     & --           & --                                  & --                           & 73.5            & 65.2            & 64.2            & 39.2            & 45.1             & 57.4\textsubscript{@200}                           & 1.0$\times$          \\
    \midrule
    \multirow{7}{*}{+AntiSD} & rev.~KL      & $0.93$                              & add.                         & 40.1            & 30.5            & 26.9            & 14.9            & 40.7             & 30.6\textsubscript{@200}                           & $\times$             \\
    \cmidrule{2-11}
                             & JSD          & none                                & add.                         & 75.5            & 68.7            & 69.2            & 45.9            & 48.7             & 61.6\textsubscript{@40}\textsuperscript{$\dagger$} & \textbf{6.7$\times$} \\
                             & JSD          & $0.90$                              & add.                         & 77.8            & 73.2            & 73.2            & \textbf{57.0}   & 48.2             & \textbf{65.9}\textsubscript{@100}                  & \textbf{6.7$\times$} \\
                             & JSD          & $0.95$                              & add.                         & 76.3            & \textbf{73.4}   & \textbf{75.8}   & 51.9            & 49.7             & 65.4\textsubscript{@180}                           & 1.4$\times$          \\
    \cmidrule{2-11}
                             & JSD          & $0.93$                              & mult.                        & 73.1            & 60.6            & 61.7            & 38.5            & 44.6             & 55.7\textsubscript{@130}                           & $\times$             \\
                             & JSD          & $0.93$                              & add.\textsuperscript{$\ast$} & 78.1            & 72.7            & 73.2            & 52.3            & \textbf{50.9}    & 65.4\textsubscript{@40}                            & \textbf{6.7$\times$} \\
    \rowcolor{antisdrow}{}   & {JSD}        & {0.93}                              & {add.}                       & \textbf{78.4}   & \textbf{73.4}   & 73.7            & 54.4            & 48.5             & 65.7\textsubscript{@180}                           & 5.0$\times$          \\
    \bottomrule
  \end{tabular}
  \par\smallskip
  {\footnotesize $\dagger$ The no-gate run on Qwen3-8B collapses by step $\sim 90$ (cf.\ Section~\ref{sec:exp:dyn}); the reported peak is from $8$ pre-collapse checkpoints. $\ast$ Gate signal swapped from teacher- to student-perplexity.}
\end{table}

\subsection{Continual AntiSD on Qwen3-4B-IT-2507}
\label{app:continual-q4}

We repeat the continual experiment from Section~\ref{sec:exp:continual} on Qwen3-4B-IT-2507. Unlike Qwen3-8B, where continual AntiSD essentially closes the gap to from-base AntiSD ($65.0$ vs $65.7$ Avg), continual on the smaller model peaks at $61.9$ briefly at step $+20$ before drifting to a plateau of $\approx 60.5$; the plateau is $2.3$\,pp short of the from-base peak ($62.8$).

\begin{table}[h]
  \caption{\textbf{Continual AntiSD on Qwen3-4B-IT-2507.} Same setup as Table~\ref{tab:continual} on the smaller model. \textbf{Bold} = column best.}
  \label{tab:continual-q4}
  \centering
  \small
  \setlength{\tabcolsep}{6pt}
  \renewcommand{\arraystretch}{1.0}
  \begin{tabular}{l|c|ccccc|c|c}
    \toprule
    \textbf{Method}                       & \textbf{Steps} & \textbf{AIME24} & \textbf{AIME25} & \textbf{AIME26} & \textbf{HMMT25} & \textbf{Minerva} & \textbf{Average}                  & \textbf{Speedup}      \\
    \midrule
    GRPO                                  & 200            & 67.8            & 57.7            & 63.5            & 34.1            & 33.2             & 51.3\textsubscript{@200}          & 1.0$\times$           \\
    +AntiSD                               & 200            & \textbf{76.6}   & \textbf{70.2}   & \textbf{74.4}   & \textbf{46.7}   & 46.4             & \textbf{62.8\textsubscript{@100}} & \textbf{10.0$\times$} \\
    \rowcolor{antisdrow} +AntiSD$^{\dag}$ & +50            & 76.2            & 69.7            & 73.3            & 43.2            & \textbf{47.0}    & 61.9\textsubscript{@20}           & \textbf{10.0$\times$} \\
    \bottomrule
  \end{tabular}
\end{table}

The plateau is consistent with two interpretations: (i) GRPO's basin admits some, but not all, of the deliberation pressure that AntiSD applies from base; (ii) the auto-recalibrated gate threshold derived from the saturated policy's $H_{\mathrm{warm}}$ is mildly conservative on this model, leaving residual AntiSD gain unrealised. We do not attempt to disentangle these here; the practical takeaway is that continual AntiSD provides a strong fraction of the from-base improvement at a fraction of the cost, with a model-conditional ceiling.

\section{Limitations and Broader Impacts}
\label{app:impacts}

\paragraph{Scope and extensions.} The conditional-PMI account in Section~\ref{sec:antisd:direction} is a local, per-step characterization of the per-token signal rather than a global-optimum statement about the combined objective; understanding the long-horizon dynamics under the full ascent + gate update is itself an interesting question. Our evaluation spans five language models from the Qwen3 and Olmo-3 families ($4$B--$30$B parameters) on mathematical reasoning, with an initial probe on code reasoning (Section~\ref{sec:exp:main}, Table~\ref{tab:code}). Several natural extensions follow the same algorithmic skeleton (Algorithm~\ref{alg:antisd}) without modifying the AntiSD update: \emph{(i)} multi-turn agentic settings, where the reward depends on a sequence of tool calls rather than a single rollout, with the privileged context spanning the full interaction trace; \emph{(ii)} broader code-reasoning benchmarks such as LiveCodeBench~v6 with longer per-problem horizons and richer test-case structure; and \emph{(iii)} richer privileged-context content -- process-level critiques, partial-credit annotations, and rationale-comparison rankings -- replacing the binary or continuous correctness feedback used here. Larger model scales beyond $30$B and multimodal conditioning are also natural settings to test whether the conditional-PMI characterization remains the dominant credit-assignment signal.

\paragraph{Broader impacts.} AntiSD is a post-training method that improves credit assignment in RLVR. Positive impacts: stronger open-weight reasoning models, lower training cost ($2$ to $10\times$ fewer steps to reach a given accuracy), and a clearer theoretical handle on why default self-distillation under-performs on math reasoning, which may inform future training-free PRM designs. Negative impacts: as with any improvement to LLM reasoning, gains are dual-use; a stronger reasoning model can be applied to adversarial or harmful tasks. AntiSD does not introduce a new attack surface beyond the pre-existing dual-use profile of large language models, and we do not release new high-risk model artifacts. We see no specific path to fairness, privacy, or safety harms beyond those already attaching to the underlying open-weight base models.

\end{document}